\documentclass[a4paper,UKenglish,cleveref, autoref, thm-restate]{lipics-v2021}
\pdfoutput=1 %uncomment to ensure pdflatex processing (mandatatory e.g. to submit to arXiv)
\hideLIPIcs  %uncomment to remove references to LIPIcs series (logo, DOI, ...), e.g. when preparing a pre-final version to be uploaded to arXiv or another public repository

\title{When Can Memorization Improve Fairness?} %TODO Please add

\author{Bob Pepin\footnote{corresponding author}}{Department of Computer Science, University of Copenhagen, Denmark}{bope@di.ku.dk}{}{}
\author{Christian Igel}{Department of Computer Science, University of Copenhagen, Denmark }{igel@di.ku.dk}{}{}
\author{Raghavendra Selvan}{Department of Computer Science, University of Copenhagen, Denmark}{raghav@di.ku.dk}{}{}
\authorrunning{B. Pepin, C. Igel and R. Selvan} %TODO mandatory. First: Use abbreviated first/middle names. Second (only in severe cases): Use first author plus 'et al.'

\Copyright{Bob Pepin, Christian Igel and Raghavendra Selvan} %TODO mandatory, please use full first names. LIPIcs license is "CC-BY";  http://creativecommons.org/licenses/by/3.0/


 %TODO mandatory; please add comma-separated list of keywords

\category{} %optional, e.g. invited paper

\relatedversion{} %optional, e.g. full version hosted on arXiv, HAL, or other respository/website
\nolinenumbers %uncomment to disable line numbering

\EventEditors{John Q. Open and Joan R. Access}
\EventNoEds{2}
\EventLongTitle{42nd Conference on Very Important Topics (CVIT 2016)}
\EventShortTitle{CVIT 2016}
\EventAcronym{CVIT}
\EventYear{2016}
\EventDate{December 24--27, 2016}
\EventLocation{Little Whinging, United Kingdom}
\EventLogo{}
\SeriesVolume{42}
\ArticleNo{23}
\usepackage{dsfont}

\newcommand{\PP}{\mathbb{P}}

\newcommand{\one}{\mathds{1}}
\newcommand{\defeq}{:=}
\newcommand{\defeqrev}{=:}

\newcommand{\yhat}{{\hat{y}}}
\newcommand{\Yhat}{{\hat{Y}}}

\newcommand{\deltasp}{\operatorname{{\Delta}s.p.}}
\newcommand{\deltaeqopp}{\operatorname{{\Delta}eq.opp.}}
\newcommand{\deltaeqodd}{\operatorname{{\Delta}eq.odds}}

\begin{document}

\maketitle

\begin{abstract}
We study to which extent additive fairness metrics (statistical parity, equal opportunity and equalized odds) can be influenced in a multi-class classification problem by memorizing a subset of the population. We give explicit expressions for the bias resulting from memorization in terms of the label and group membership distribution of the memorized dataset and the classifier bias on the unmemorized dataset. We also characterize the memorized datasets that eliminate the bias for all three metrics considered. Finally we provide upper and lower bounds on the total probability mass in the memorized dataset that is necessary for the complete elimination of these biases.
\end{abstract}

\section{Introduction}
We are interested in the fairness properties of a classifier as measured by metrics such as statistical parity, equal opportunity or equalized odds \cite{barocas2023fairness}. These metrics assume that the population can be partitioned into groups based on the value of a protected attribute, which we refer to simply as \emph{groups} for the rest of the paper. Each of these metrics only considers the average performance of the classifier on each group and therefore cannot tell the difference whether the fairness is reduced by moderately increasing the performance uniformly over all individuals in an underperforming group or whether we selected a special subgroup within a group on which we produce an extreme increase in performance. We refer to the later practice as \emph{gaming} the metric, since it does improve the metric but is most likely not the effect that was intended by imposing fairness constraints. The purpose of the current work is to examine the extent to which a fairness metric can be gamed in an extreme setting called {\em memorization}. Memorization in this work refers to the case where a classifier provides perfect predictions on a memorized subpopulation (which can include data from multiple groups and labels).

% , which is closely related to the notion of memorization formalized by Feldman in~\cite{feldman2020does}

The fairness metrics studied in this paper are statistical parity, equal opportunity and equalized odds. We seek to characterize as precisely as possible how much of the total probability mass of the data distribution  needs to be memorized so that the memorization can compensate for a bias in the original classifier. All our results are formulated in terms of the relative proportions of different groups and labels in the whole population and in the memorized subpopulation.

In summary, our contributions are:
\begin{enumerate}
    \item For each fairness metric considered, we express the bias of a memorizing classifier in terms of the composition of the total population and the memorized subpopulation as well as the bias of the classifier on the unmemorized subpopulation (Theorem~\ref{theorem:sp}).
    \item We characterize the compositions of the memorized subpopulations that result in zero bias of the memorizing classifier in terms of systems of linear equations (Theorems~\ref{theorem:deltaspzero},~\ref{theorem:deltaeqoppzero} and~\ref{theorem:deltaeqoddzero}).
    \item We provide upper and lower bounds on the probability mass that needs to be memorized to compensate for a given statistical parity or equal opportunity bias (Corollaries~\ref{corollary:spsufficient},~\ref{corollary:spnecessary},~\ref{corollary:eqoppsufficient},~\ref{corollary:eqoppnecessary}).
\end{enumerate}

% We study the fairness properties of a classification problem where a base classifier has been augmented with a memorizing classifier so that it has perfect prediction on a subset of the instances. 

% We are studying additive fairness metrics. In this setting, it is possible to compensate for the mistreatment of some individuals within a group by a better treatment of other individuals within the same group. Our results provide insights into the situations where this is possible or could possibly have happened, by providing theoretical bounds under the arguably best treatment of perfect prediction accuracy.

% Most previous work has studied binary classification, we focus on multiple classes.

% In this work, we assume that the classifier predicts perfectly on the memorized dataset. This should be interpreted as a simplified model of a classifier that has much higher accuracy on the memorized subset compared to the non-memorized subset.

To highlight the applicability of our results to practical situations involving fairness, we proceed with a list of examples and show how our results can be used to improve or estimate bias.

\begin{example}[Dataset Compression]
    Our framework can be applied to a dataset compression problem for a nearest neighbour classifier, where only a subset of data points are selected for fitting the classifier (strictly smaller than the size of any reasonable coreset~\cite{huang2019coresets}). The classifier will have perfect prediction on the points for which the fitted points form a coreset, corresponding to the memorized dataset, and without additional assumptions will in a first approximation predict randomly according to the label distribution in the global population on the rest. Theorems~\ref{theorem:deltaspzero} and~\ref{theorem:deltaeqoppzero} quantify how the distribution of labels and groups in the fitted subset (and its complement) affects the bias.
\end{example}

\begin{example}[Model Development]
Imagine you are a data scientist tasked with engineering a prompt to tailor a large language model for a specific use case. To do so, you use a combination of a small development dataset and a larger evaluation dataset. You iterate on the development dataset by manually tracing failure cases and adjusting the prompt in return. Once in a while you evaluate your prompts on the evaluation dataset to confirm that the modifications generalize well. You will likely end up with a model that performs perfectly on the development dataset and good enough on the rest. In such a situation, the results in Theorems~\ref{theorem:deltaspzero} and~\ref{theorem:deltaeqoppzero} tell you how a bias in the choice of the development dataset translates into a bias of the resulting model.
\end{example}

\begin{example}[Model Improvement]\label{example:focus}
    Another example similar to the previous one could be the improvement of a complex machine learning pipeline, where time and money have to be allocated to focus model improvement efforts on a subset of the input space. In our framework the focused subset corresponds to the memorized dataset and the bias of the original classifier is known empirically. Theorems~\ref{theorem:deltaspzero} and~\ref{theorem:deltaeqoppzero} provide guidance on how to choose this focused subset in order to avoid introducing too much additional biases due to the model performance improving substantially only for a subset of the space.
\end{example}

% \begin{example}[Model Fine-tuning] 
% % \TBD
%     Consider the task of fine-tuning a pretrained model in order to overcome the biases in the original dataset. Theorem~\ref{theorem:sp} offers a recipe for an iterative method where one can start with a pretrained model, and augment it with a memorizer during the fine-tuning process to obtain a model with lower bias than the pretrained model. The results in Theorem~\ref{theorem:sp} can be used to curate the fine-tuning dataset that can be used to reduce bias in the fine-tuned model.
% \end{example}

\begin{example}[Fairness Requirements]
%    It is well-known that additive fairness metrics can be gamed by compensating for the mistreatment of some individuals in a group via a radically better treatment of a small subgroup within that same group (see for example \cite{mosse_multiplicative_2023}). 
    
    Imagine you are an organization that plans to tender the development of a machine learning model with the requirement that it satisfies statistical parity or equal opportunity between certain groups. You want to know whether you need to include additional safeguards against gaming of the fairness metric, for example, by routing predictions for a certain subgroup to a specialized model, which would satisfy the formal requirements but cause a large disparity within a group. Our results in Corollaries \ref{corollary:spnecessary} and \ref{corollary:eqoppnecessary} provide lower bounds on which fraction of the total population the specialized model would have to cover in order to hide a certain level of intrinsic bias. This can be used to assess how easy and thus attractive it would be for a contractor to develop such a specialized model to make an unfair model seem fair. In this context, our setting of memorization, corresponding to a perfect model, can be viewed as an extreme case of an exceptionally good model.
\end{example}

\subparagraph{Related work.}
%Adjusting the ratio of minority subpopulations within a biased dataset has been a common method to mitigate bias in machine learning models~\cite{kamiran2009classifying}.  
Formal fairness metrics were first introduced in the domain of fair testing in the statistics community in the 1960s and 1970s~(see \cite{hutchinson_50_2019}) and were further popularized in the machine learning community by Dwork et al.~\cite{dwork2012fairness}. This line of research brought forward the idea that it is possible to partition the population to be tested/predicted into different groups based on a so-called sensitive attribute, which forms the basis for all popular fairness metrics.

The approaches that are closest to the present work are of empirical nature. The FairBatch method~\cite{roh_fairbatch_2021} selects training data based on a fairness metric, thus focusing the model on a subset of the data space similar to Example~\ref{example:focus} above. The fairness in our models is also related to fairness in pipelines~\cite{bower_fair_2017} as it can be viewed as the composition of two models (memorization and a ``base'' classifier). Both studies focus on the development of algorithms to reduce certain fairness metrics for practical applications, which is different from the present study which addresses the theoretical limitations in a simplified model.

A number of works investigate under which conditions models naturally end up memorizing parts of the dataset. The experiments in Chang et al.~\cite{chang2021privacy} show that models memorize a significantly larger portion of the dataset when equal opportunity constraints are added. This suggests that common training procedures can game the fairness metrics through memorization in the sense studied in the present paper. In~\cite{feldman2020does}, Feldman  explores how long-tail distributions force machine learning models to memorize in order to generalize well.

Mossé~\cite{mosse_multiplicative_2023} discusses the issue of gaming additive fairness metrics, but does not quantify the magnitude of the effect or under which conditions this is possible.

Going beyond a strict partitioning into groups, in Ilvento~\cite{ilvento2020metric} the author proposes a framework that goes beyond group-based fairness metrics. She emphasizes the importance of context-specific solutions and the need for continuous evaluation and adaptation of fairness metrics to reflect evolving societal norms and values. This is also echoed in the commentaries from Ricci et al.~\cite{ricci2022addressing} and Holm et al.~\cite{holm2023bias} with a focus on critical application domains such as healthcare. 

To the best of the authors' knowledge, no prior work has studied the interplay between memorization and fairness with specific prescriptions on the extent of memorization required to reduce bias in machine learning models. Our work aims to fill this gap.

\section{Setting and Notation}

We model our problem using probability theory, where each realization of the (jointly distributed) random variables corresponds to an instance/data point of the input space. The random variable $Y$ denotes a label in $[K]$ (integers between $1$ and $K$) and $\Yhat$ denotes the output of our predictor, also taking values in $[K]$. Each instance can also have a sensitive attribute, modelled by a random variable $A$ so that $A = 1$ if the sensitive attribute is present and $A = 0$ otherwise. We model the memorization effect by a random variable $D$ so that $D = 1$ if an instance is memorized and $D = 0$ otherwise. Our results do not require us to  represent the input space explicitly.

We assume that the classifier is memorizing on a subset of the data in the sense that
\begin{equation}
    \PP(\Yhat = Y \mid D = 1) = 1
\end{equation}
and we denote $p_D$ the proportion of the input space that has been memorized:
%\footnote{The memorization definition in Feldman~\cite{feldman2020does} defines a gap in prediction accuracy, which in our notations would correspond to $\PP(\Yhat = Y \mid D = 1) - \PP(\Yhat = Y \mid D = 0)$}
\begin{equation}
    p_D \defeq \PP(D = 1).
\end{equation}

The first fairness measure we consider is \emph{statistical parity}, also referred to as 
demographic parity, group fairness, and disparate impact \cite{barocas2023fairness}.
\begin{definition}[Statistical Parity]
    Statistical parity reflects bias on the level of a population. For each class $y \in [K]$ the statistical parity gap on $\yhat$ expresses how much more likely a classifier is to classify a sample as $\yhat$ if it is part of the group having $A = 1$ than if it is not. We define the \emph{statistical parity gap} on $\yhat \in [K]$ as
    \begin{equation}
        \deltasp(\hat{y}) \defeq \PP(\Yhat = \hat{y} \mid A = 1) - \PP(\Yhat = \hat{y} \mid A = 0).
    \end{equation}
    We additionally define the statistical parity gap on the unmemorized dataset as
    \begin{equation}
        \deltasp(\hat{y} \mid D=0) \defeq \PP(\Yhat = \hat{y} \mid A = 1, D = 0) - \PP(\Yhat = \hat{y} \mid A = 0, D = 0).
    \end{equation}    
\end{definition}
Furthermore, we look at \emph{equalized odds} and \emph{equal opportunity} \cite{barocas2023fairness}.
\begin{definition}[Equalized Odds]
    Equalized odds is a fairness measure that reflect bias on the level of an individual. The equalized odds gap is defined for each combination of sample label and predicted label. It quantifies the difference between the $A = 1$ and $A = 0$ population of how likely a classifier is to predict $\yhat$ if the true label was $y$. For each $\yhat \in [K], y \in [K]$ we define the equalized odds gap to be
    \begin{equation}
        \deltaeqodd(y, \hat{y}) \defeq \PP(\Yhat = \hat{y} \mid A = 1, Y = y) - \PP(\Yhat = \hat{y} \mid A = 0, Y = y).
    \end{equation}
    We also define the corresponding gap on the unmemorized dataset as
    \begin{multline}
        \deltaeqodd(y, \hat{y} \mid D = 0) \defeq \\
        \PP(\Yhat = \hat{y} \mid A = 1, Y = y, D = 0) - \PP(\Yhat = \hat{y} \mid A = 0, Y = y, D = 0).
    \end{multline}
\end{definition}
    
\begin{definition}[Equal Opportunity]
    Like equalized odds, equal opportunity reflects bias on the level of an individual. It is a coarser measure than equalized odds. The equal opportunity gap is defined for each label and it reflects how much more likely the classifier is to make the correct prediction on the $A = 1$ group than on the $A = 0$ group. We can define the equal opportunity gap in terms of the equalized odds gap as
    \begin{equation}
        \deltaeqopp(y) \defeq \deltaeqodd(y, y) = \PP(\Yhat = Y \mid A = 1, Y = y) - \PP(\Yhat = Y \mid A = 0, Y = y).
    \end{equation}
    We similarly define the equal opportunity gap on the unmemorized dataset to be
    \begin{multline}
        \deltaeqopp(y \mid D = 0) \defeq \deltaeqodd(y, y \mid D = 0) \\
        = \PP(\Yhat = Y \mid A = 1, Y = y, D = 0) - \PP(\Yhat = Y \mid A = 0, Y = y, D = 0).
    \end{multline}
\end{definition}

For a function $f$ (meaning $\deltasp, \deltaeqopp, \deltaeqodd$ in the statements of our results), we use the notation $f = 0$ to denote that $f(x) = 0$ uniformly over all $x$ in the domain of $f$.

% We also use the notation $\deltasp \in \mathbb{R}^K$ for the vector $(\deltasp(y))_{y\in [K]}$ and similarly for $\deltaeqopp$ as well as $\deltaeqodd \in \mathbb{R}^{K\times K}$ for the matrix $(\deltaeqodd(y, \yhat))_{y \in [K], \yhat \in [K]}$ so that $\deltaeqodd = 0$ for example means equal to zero uniformly over all $y, \yhat$.

We introduce the following notation for the distribution of different subpopulations in the data space and in the memorized subset:
\begin{align}
p^+ &\defeq \PP(A = 1), & q^+ &\defeq \PP(A = 1 \mid D = 1), \notag\\
p_y &\defeq \PP(Y = y), & q_y &\defeq \PP(Y = y \mid D = 1), \notag\\ 
p_y^+ &\defeq \PP(Y = y, A = 1) & q_y^+ &\defeq \PP(Y = y, A = 1 \mid D = 1)
\end{align}
We  also define the corresponding quantities for the $A = 0$ group, $p^-$, $p_y^-$, $q^-$, and $q_y^-$, by replacing $A = 1$ with $A = 0$ in the definitions. Then we get the relations $p^- = 1 - p^+, p_y^- = p_y - p_y^+, q^- = 1 - q^+, q_y^- = q_y - q_y^+$ which are helpful to interpret some of the formulas in the results. It can be verified that all of our results transform in the appropriate way when exchanging the role of the sensitive and non-sensitive subpopulations.

We also introduce notations for the label statistics and confusion matrices of the base classifier:
\begin{align}
    \varphi_y^+ &\defeq \PP(\Yhat = y \mid D = 0, A = 1), & C_{y,\yhat}^+ &\defeq \PP(\Yhat = \yhat \mid D = 0, Y = y, A = 1), \notag\\
    \varphi_y^- &\defeq \PP(\Yhat = y \mid D = 0, A = 0), & C_{y,\yhat}^- &\defeq \PP(\Yhat = \yhat \mid D = 0, Y = y, A = 0)
\end{align}
Note that $\deltasp(\hat{y} \mid D=0) = \varphi^+_\yhat - \varphi^-_\yhat$ and $\deltaeqodd(y, \hat{y} \mid D = 0) = C^+_{y,\yhat} - C^-_{y,\yhat}$.

\section{Results}
In this section, we show how the fairness gap of the classifier with memorization varies with the composition of the memorized dataset in terms of subpopulations with different labels and sensitive attribute values. It should be noted that the fairness gaps of the base classifier ($\varphi^+_y, C^+_{y,\yhat}$) can depend on the choice of the memorized dataset, so that the results presented here do not directly yield a method for eliminating the bias without some extra assumptions. The first theorem could form the basis of an iterative method, whereas the following theorems could be used with perturbation bounds. The development of numerical methods is left for follow-up work and outside of the scope of the present study. All the proofs are long but elementary and can thus be found in the appendix.

The following theorem expresses the different fairness gaps of the combined ensemble of memorizer and base classifier in terms of the bias of the base classifier and a correction term involving the parameters $q_y^+$ and $q_y$ (note that $q^+ = \sum_y q_y^+$). All gaps scale linearly with $p_D$, meaning that as expected the influence of the memorization increases as we memorize more data and we recover the base classifier bias as $p_D$ goes to $0$. They also scale inversely proportional to $p^+(1-p^+)$ (which reaches its maximum $1/4$ at $p^+=1/2$), showing that the influence of the correction increases with the level of imbalance between the two groups in the whole population.

As expected, the correction for statistical parity involves the base classifier through $\varphi_y^+$ which only depends on the probability of the base classifier to predict a given label, irrespective of the true label of the sample. The expression for equal opportunity involves the misclassification rate for each class $1-C^+_{y,y}$, consistent with the definition of equal opportunity, whereas to compute the equalized odds for the ensemble we need access to the full confusion matrix of the base classifier $C_{y,\yhat}$.

\begin{theorem}\label{theorem:eqodds}\label{theorem:sp}
For any $y, \hat{y} \in [K]$ we have the following expressions for the different fairness gaps:
\begin{align}
    \deltasp(\hat{y}) = &\frac{p_D}{p^+(1-p^+)} \bigl(\varphi_\yhat^+ (p^+ - q^+) - (q_\yhat p^+ - q^+_\yhat) \bigr) \notag\\ 
    & + \deltasp(\yhat \mid D=0) \left(1 - p_D \frac{1-q^+}{1-p^+}\right), \\
    \deltaeqopp(y) 
    =& \frac{ p_D }{p_y^+ (p_y-p_y^+)} \left(C_{y,y}^+(h) - 1 \right)(q_y p_y^+ - p_y q_y^+) \notag\\
    &+ \deltaeqopp(y \mid D=0)\left( 1 - p_D \frac{q_y-q_y^+}{p_y-p_y^+}\right), \\   
    \deltaeqodd(y, \yhat) 
    =& \frac{ p_D }{p_y^+ (p_y-p_y^+)} \left(C_{y,\yhat}^+(h) - \one\{y=\yhat\} \right)(q_y p_y^+ - p_y q_y^+) \notag\\
    & + \deltaeqodd(y,\yhat \mid D=0)\left( 1 - p_D \frac{q_y-q_y^+}{p_y-p_y^+}\right),
\end{align}
assuming that $p^+ \in (0, 1)$.
\end{theorem}

The next result completely characterizes the values of $(q_y, q_y^+)$ that compensate for the bias of a base classifier in terms of the bias of the prediction rates of the base classifier $\varphi_y^+$ and $\varphi_y^-$. If $\varphi_y^+$ and $\varphi_y^-$ are known and independent of $(q_y, q_y^+)$ then this set of linear inequalities (a linear program with constant objective) can be numerically solved to obtain an unbiased classifier.
% In the case that $\varphi^+_y$ and $\varphi^-_y$ depend on the parameters, we still obtain a linear program after linearizing $\varphi^+_y$ and $\varphi^-_y$ around some fixed value EXCEPT for \sum_{y'} q_{y'}^+ c_y
Note that the set of solutions can be empty depending on the values of $p_D$ and the composition of the input population.

\begin{theorem}\label{theorem:deltaspzero}
    The solutions to $\deltasp = 0$ are exactly the values of $(q_y, q_y^+)$ such that for all $y \in [K]$ we have
    \begin{gather}
        p^+ q_y - q_y^+ - \left(p^+ - \sum_{y'=1}^K q_{y'}^+\right)c_y - b_y = 0,\\
        \sum_y q_y = 1, \\
        0 \leq q_y^+ \leq q_y
    \end{gather}
    with
    \begin{align}
        b_y &\defeq p_y^+ (1-p_y^+) \frac{1-p_D}{p_D} (\varphi_y^+ - \varphi_y^-),\\
        c_y &\defeq (1-p_y^+) \varphi_y^+ + p_y^+ \varphi_y^-.
    \end{align}
    Furthermore the set of solutions is non-empty if and only if
    \begin{multline}
        \{ x \in \mathbb{R}^{K+1}\colon 
        -(1-p^+)x_i + c_i + x_{K+1} \leq 0 \text{ and }
        p^+ x_i + x_{k+1} \leq 0 \, \forall i \in [K] \} \\
        \subset \left\{ x \in \mathbb{R}^{K+1}\colon \sum_{i=1}^K (p^+ c_i + b_i) x_i + x_{K+1} \leq 0\right\}.
    \end{multline}
\end{theorem}

\begin{corollary}\label{corollary:spsufficient}
    The set of solutions is non-empty if
    \begin{align}
        p_D &\geq (1-p^+) \max_y \frac{\varphi^-_y - \varphi^+_y}{\varphi^-_y}.
    \end{align}
    % or if
    % \begin{align}
    %     p_D &\geq p^+ \max_y \frac{\varphi^+_y - \varphi^-_y}{\varphi^+_y}
    % \end{align}
\end{corollary}

\begin{corollary}\label{corollary:spnecessary}
    The set of solutions is empty if
    \begin{align}
        p_D &< (1-p^+) \min_y \frac{\varphi^-_y - \varphi^+_y}{\varphi^-_y}.
    \end{align}
    % or
    % \begin{align}
    %     p_D &< p^+ \min_y \frac{\varphi^+_y - \varphi^-_y}{\varphi^+_y}
    % \end{align}
\end{corollary}

Similar to the previous theorem, the following results exactly characterizes the parameter values that eliminate the equal opportunity bias in terms of the misclassification rates $1-C^+_{y,y}$ and $1-C^-_{y,y}$.

\begin{theorem}\label{theorem:deltaeqoppzero}
    The solutions to $\deltaeqopp = 0$ are exactly the values of $(q_y, q_y^+)$ such that for all $y \in [K]$ we have
    \begin{gather}
        q_y = \frac{p_y}{p_D} - \alpha_y \lambda_y, \\
        q_y^+ = \frac{p_y^+}{p_D} - \frac{p_y^+}{1-C_{y,y}^+} \lambda_y, \\
        \sum_y \alpha_y \lambda_y = \frac{1-p_D}{p_D}, \\
        \lambda_y \leq \frac{1-C_{y,y}^+}{p_D}, \\
        \lambda_y \leq \frac{1-C_{y,y}^-}{p_D}
    \end{gather}
    with
    \begin{align}
        \alpha_y &\defeq \frac{p_y^+}{1-C_{y,y}^+} + \frac{p_y^-}{1-C_{y,y}^-}.
    \end{align}
\end{theorem}

\begin{corollary}\label{corollary:eqoppsufficient}
    The set of solutions is non-empty if
    \begin{align}
        p_D &\geq \sum_y p_y^- \frac{C_{y,y}^+ - C_{y,y}^-}{1-C_{y,y}^-} \text{ and } \\
        p_D &\geq \sum_y p_y^+ \frac{C_{y,y}^- - C_{y,y}^+}{1-C_{y,y}^+}.
    \end{align}
    From this it follows that the set of solutions is non-empty if
    \begin{align}
        p_D &\geq p^- \max_y \frac{C_{y,y}^+ - C_{y,y}^-}{1-C_{y,y}^-} \text{ and } \\
        p_D &\geq p^+ \max_y \frac{C_{y,y}^- - C_{y,y}^+}{1-C_{y,y}^+}.
    \end{align}
\end{corollary}

\begin{corollary}\label{corollary:eqoppnecessary}
    The set of solutions is empty if either
    \begin{align}
        p_D &< p^- \min_y \frac{C_{y,y}^+ - C_{y,y}^-}{1-C_{y,y}^-} \text{ or }\\
        p_D &< p^+ \min_y \frac{C_{y,y}^- - C_{y,y}^+}{1-C_{y,y}^+}.
    \end{align}
\end{corollary}

Equalized odds is a very strong constraint on the classifier. The following result shows that the corresponding bias can be completely eliminated by memorization only in very specific situations. Even if this is unlikely to occur in practical applications, we expect that the result can still be useful in combination with perturbation methods.

\begin{theorem}\label{theorem:deltaeqoddzero}
    The equation $\deltaeqodd=0$ has a solution only if the ratio
    \begin{equation}
        r_y \defeq \frac{1-C_{y,\yhat}^+}{1-C_{y,\yhat}^-}, \quad \yhat \neq y
    \end{equation}
    is independent of $\yhat$ for all $\yhat \neq y$. In that case $r_y$ is given by
    \begin{equation}
        r_y = \frac{(K-1) - (1-C_{y,y}^+)}{(K-1) - (1-C_{y,y}^-)}.
    \end{equation}
    In that case the only values of $p_D, q_y$ and $q_y^+$ for which $\deltaeqodd = 0$ are given by
    \begin{align}
    p_D &= \sum_y \frac{(p_y^+ + r_y p_y^-)(C_{y,y}^- - C_{y,y}^+)}{(1-C_{y,y}^+) - r_y (1-C_{y,y}^-)}, \\
    q_y^+ &= \frac{p_y^+}{p_D}\frac{(C_{y,y}^- - C_{y,y}^+)}{(1-C_{y,y}^+) - r_y (1-C_{y,y}^-)}, \\
    q_y &= \frac{p_y^+}{p_D}\frac{(p_y^+ + r_y p_y^-)(C_{y,y}^- - C_{y,y}^+)}{(1-C_{y,y}^+) - r_y (1-C_{y,y}^-)}.
    \end{align}
\end{theorem}

% It is well-known in the fairness literature (cite TBD) that group-level and individual-level fairness are in general conflicting objectives. The following result shows under which conditions a classifier with memorization can compensate for both kinds of biases simultaneously.

% \begin{theorem}\label{theorem:simultaneous}
% The conditions $\deltasp = 0$ and $\deltaeqopp = 0$ can only hold simultaneously if
% \begin{gather*}
%     TBD
% \end{gather*}
% \end{theorem}

\section{Conclusion}
We have analyzed how statistical parity, equal opportunity and equalized odds metrics can be improved by memorizing parts of the input space. This 
%is usually not desired and 
can lead to disparities within groups, and unfortunately the effect has been observed in empirical studies when training machine learning models with fairness constraints (see related works section).

All of our results are framed as systems of linear equality and inequality constraints which can be numerically solved by standard linear programming algorithms such as simplex or interior point methods.

When memorization is 
%inevitable due to process and workflow constraints
utilized, our results show how we can avoid introducing additional bias through a clever choice of the parts of the dataset to be memorized.

Knowing the limits of memorization should be helpful in developing future methods that reduce bias without causing large disparities within groups through memorization. Our results precisely characterize the composition of the memorized dataset in terms of distribution of group memberships and labels. We also provide upper and lower bounds on how much of the data space would have to be memorized to obtain zero bias solely through memorization. 

The results in this work could in principle be used for the development of numerical methods for training fair classifiers according to the metrics studied here. However, care should be taken that this does not result in methods that game the metrics by default, leading to discrimination of subgroups within the groups.

\subsubsection*{Acknowledgements} We thank Julian Schön 
%and Nirupam Gupta 
for useful discussions.
The authors acknowledge funding received under European Union’s Horizon Europe Research and Innovation programme under grant agreements No.~101070284 and No.~101070408. CI cknowledges support by the Pioneer Centre for AI, DNRF grant number P1.

\bibliography{local_ref,references_raghav}

\appendix
\section{Proofs}

\begin{proof}[Proof of Theorem~\ref{theorem:sp}]
    Let $\Yhat = h(X)$. We start by showing the equality for $\deltasp$. We have
    \begin{equation}
        \PP(\Yhat = \yhat \mid A = 1) = \PP(\Yhat = \yhat, D = 1 \mid A = 1) + \PP(\Yhat = \yhat, D = 0 \mid A = 1).
    \end{equation}
    For the first term,
    \begin{align}
        \PP(\Yhat = \hat{y}, D = 1 \mid A = 1)
        &= \PP(Y = \hat{y}, D = 1 \mid A = 1) \\
        &= \PP(Y = \hat{y}, A = 1 \mid D = 1) \PP(D = 1) / \PP(A = 1) \\
        &= p_D \frac{q_\yhat^+}{p^+}
    \end{align}
    since $\Yhat = Y$ if $D = 1$ by assumption.
    By an identical argument
    \begin{equation}
        \PP(\Yhat = \hat{y}, D = 1 \mid A = 0) = p_D \frac{q_\yhat^-}{p^-}
        = p_D \frac{q_\yhat-q_\yhat^+}{1-p^+}
    \end{equation}  
    For the second term,
    \begin{align}
        \PP(\Yhat = \hat{y}, D = 0 \mid A = 1)
        &= \PP(\Yhat = \yhat \mid D = 0, A = 1) (1 - \PP(D = 1 \mid A = 1)) \\
        &= \PP(\Yhat = \yhat \mid D = 0, A = 1) \left(1 - \frac{\PP(A = 1 \mid D = 1) \PP(D = 1)}{\PP(A = 1)}\right) \\
        &= \varphi_\yhat^+ \left(1 - p_D \frac{q^+}{p^+}\right).
    \end{align}
    and
    \begin{multline}
        \PP(\Yhat = \hat{y}, D = 0 \mid  A = 0) \\
        = \varphi_\yhat^- \left(1 - p_D \frac{1-q^+}{1-p^+}\right) = \varphi_\yhat^+ \left(1 - p_D \frac{1-q^+}{1-p^+}\right) - (\varphi_\yhat^+ - \varphi_\yhat^-) \left(1 - p_D \frac{1-q^+}{1-p^+}\right).
    \end{multline}
    Now the result follows since
    \begin{align}
        \deltasp(\yhat) &= \PP(\Yhat = \hat{y}, D = 1 \mid A = 1) - \PP(\Yhat = \hat{y}, D = 1 \mid A = 0) \notag\\ & \quad + \PP(\Yhat = \hat{y}, D = 0 \mid A = 1) - \PP(\Yhat = \hat{y}, D = 0 \mid A = 0) \\
        &= \frac{p_D}{p^+(1-p^+)} (q_\yhat^+ - q_y p^+) 
        + \varphi_\yhat^+ \left(1 - p_D \frac{q^+}{p^+}\right)
        - \varphi_\yhat^- \left(1 - p_D \frac{1-q^+}{1-p^+}\right)
    \end{align}
    which is equal to the result since $\varphi_\yhat^+ - \varphi_\yhat^- = \deltasp(\yhat \mid D = 0)$.

    Regarding $\deltaeqodd$ we have
        \begin{multline}
        \PP(\Yhat = \yhat \mid A = 1, Y = y) \\= \PP(\Yhat = \yhat, D = 1 \mid A = 1, Y = y) + \PP(\Yhat = \yhat, D = 0 \mid A = 1, Y = y).
    \end{multline}
    For the first term,
    \begin{align}
        \lefteqn{\PP(\Yhat = \hat{y}, D = 1 \mid A = 1, Y = y)}\quad \notag\\
        &= \PP(Y = \hat{y}, D = 1 \mid A = 1, Y = y) \\
        &= \one\{y = \hat{y}\} \PP(D = 1 \mid A = 1, Y = y) \\
        &= \one\{y = \yhat\} \PP(A = 1, Y = y \mid D = 1) \PP(D = 1) / \PP(A = 1, Y = y) \\
        &= p_D \one\{y = \yhat\} \frac{q_y^+}{p_y^+}
    \end{align}
    using that $\Yhat = Y$ on $\{D = 1\}$ by assumption for the first equality. By an identical argument
    \begin{equation}
        \PP(\Yhat = \hat{y}, D = 1 \mid A = 0, Y = y) = p_D \one\{y = \yhat\} \frac{q_y^-}{p_y^-}
        = p_D \one\{y = \yhat\} \frac{q_y-q_\yhat^+}{p_y-p^+}.
    \end{equation}
    For the second term we have
    \begin{align}
        \lefteqn{\PP(\Yhat = \hat{y}, D = 0 \mid A = 1, Y = y)}\quad \notag\\
        &= \PP(\Yhat = \yhat \mid D = 0, A = 1, Y = y) (1 - \PP(D = 1 \mid A = 1, Y = y)) \\
        &= \PP(\Yhat = \yhat \mid D = 0, A = 1, Y = y) \left(1 - \frac{\PP(A = 1, Y = y \mid D = 1) \PP(D = 1)}{\PP(A = 1, Y = y)}\right) \\
        &= C_{y,\yhat}^+ \left(1 - p_D \frac{q_y^+}{p_y^+}\right).
    \end{align}
    and again by symmetry
    \begin{multline}
        \PP(\Yhat = \hat{y}, D = 0 \mid  A = 0, Y = y) = C_{y,\yhat}^- \left(1 - p_D \frac{q_y-q_y^+}{p_y-p_y^+}\right) \\ = C_{y,\yhat}^+ \left(1 - p_D \frac{q_y-q_y^+}{p_y-p_y^+}\right) - (C_{y,\yhat}^+ - C_{y,\yhat}^-) \left(1 - p_D \frac{q_y-q_y^+}{p_y-p_y^+}\right).
    \end{multline}
    Now the result for $\deltaeqodd$ follows since
    \begin{align}\label{eq:deltaspplusminus}
        \lefteqn{\deltaeqodd(\yhat)} \quad\notag \\
        &= \PP(\Yhat = \hat{y}, D = 1 \mid A = 1, Y = y) - \PP(h(X) = \hat{y}, D = 1 \mid A = 0, Y = y) \notag\\ & \quad + \PP(\Yhat = \hat{y}, D = 0 \mid A = 1, Y = y) - \PP(\Yhat = \hat{y}, D = 0 \mid A = 0, Y = y) \\
        & = C_{y,\yhat}^+ \left(1 - p_D \frac{q_y^+}{p_y^+}\right) - C_{y,\yhat}^- \left(1 - p_D \frac{q_y^-}{p_y^-}\right) + p_D \one\{y = \yhat\} \frac{q_y^+}{p_y^+} - p_D \one\{y = \yhat\} \frac{q_y^-}{p_y^-}.
    \end{align}
    Finally we have that $\deltaeqopp(h; y) = \deltaeqodd(h; y, y)$ which immediately yields the corresponding result.
\end{proof}

\begin{proof}[Proof of Theorem \ref{theorem:deltaspzero}]
Let $u = (q_1, \ldots, q_K), v = (q_1^+, \ldots, q_K^+), \varphi^+ = (\varphi^+_1, \ldots, \varphi^+_K), \varphi^- = (\varphi^-_1, \ldots, \varphi^-_K), 1_K = (1, \ldots, 1)$ be $K$-dimensional vectors so that
\begin{equation}
    \deltasp = f(u, v)
\end{equation}
with
\begin{equation}
    f(u, v) = \frac{p_D}{p^+(1-p^+)}(v-p^+u) 
    + \varphi^+\left(1 - p_D \frac{1_K^\top v}{p^+}\right)
    - \varphi^-\left(1 - p_D \frac{1-1_K^\top v}{1-p^+}\right).
\end{equation}

Note that $1_K^\top u = 1$ so that for $x \in \mathbb{R}^K$
\begin{align}
    \lefteqn{f(u, p^+u + x) } \quad \\
    &= \frac{p_D}{p^+(1-p^+)}x 
    + \varphi^+\left(1 - p_D \frac{p^+ + 1_K^\top x}{p^+}\right)
    - \varphi^-\left(1 - p_D \frac{1 - p^+ - 1_K^\top x}{1-p^+}\right) \\
    &= \frac{p_D}{p^+(1-p^+)}x + \varphi^+ (1-p_D) - \frac{p_D}{p^+}\varphi^+ 1_K^\top x - \varphi^- (1-p_D) - \frac{p_D}{1-p^+}\varphi^- 1_K^\top x
\end{align}
which we rewrite as
\begin{equation}
    f(u, p^+u + x) = p_D M x + (1-p_D) (\varphi^+ - \varphi^-)
\end{equation}
with
\begin{align}
    M &= \frac{1}{p^+(1-p^+)}\left(I_K - c 1_K^\top\right), \\
    c &= (1-p^+)\varphi^+ + p^+ \varphi^-.
\end{align}
Let \begin{equation}
    b = p^+(1-p^+) \frac{1-p_D}{p_D}(\varphi^+ - \varphi^-)
\end{equation}
and note that since $1_K^\top \varphi^+ = 1_K^\top \varphi^- = 1$ we have $1_K^\top b = 0$ and
\begin{equation}
    p_D M b = \frac{p_D}{p^+(1-p^+)} b = (1-p_D) (\varphi^+ - \varphi^-)
\end{equation}
so that
\begin{equation} \label{eq:Mx}
    f(u, p^+u - b + x) = 0 \text{ if and only if } M x = 0.
\end{equation}

We proceed to determine the dimension of the subspace of vectors $x$ for which \eqref{eq:Mx} holds (i.e., $\dim \ker(M)$). We ignore constraints for now and will use them later to constrain the space of solutions.
By the Weinstein–Aronszajn identity we have that
\begin{equation}
    \det(M) \propto \det(I_K - c 1_K^\top) = 1 - 1_K^\top c = 0
\end{equation}
since $1_K^\top c = 1$ so that
\begin{equation}
    \dim \ker (M) \geq 1.
\end{equation}
The matrix $M$ is of the form
\begin{equation}
        M = \frac{1}{p^+(1-p^+)} 
        \begin{pmatrix}
             -c_1 + 1 & -c_1 & \cdots & -c_1 \\
             -c_2 & -c_2 + 1 & \cdots & -c_2 \\
             \cdots & \cdots & \cdots & \cdots \\
             -c_K & -c_K & \cdots & -c_K + 1
        \end{pmatrix}.
\end{equation}
For the following we can assume without loss of generality that $c_1 \neq 0$ since the determinant is invariant to reordering of rows and columns up to a change of sign. Denoting $M_{2:,2:}$ the cofactor matrix formed from $M$ by deleting the first row and first column and $c_{2:}$ the vector formed by deleting the first row of $c$, we can compute the $(1,1)$-minor of $M$ as
\begin{equation}
    \det (M_{2:,2:}) \propto \det(I_{K-1} - 1_{K-1}^\top c_{2:}) = 1 - 1_{K-1}^\top c_{2:} = c_1 \neq 0.
\end{equation}
Since the rank of a matrix is lower-bounded by the order of any non-zero minor this shows that the rank of $M$ is at least $K-1$ and thus
\begin{equation}
    \dim \ker(M) \leq 1
\end{equation}
so that together with the lower bound we have
\begin{equation}
    \dim \ker(M) = 1.
\end{equation}
Because $1_K^\top c = 1$ we also have $M c = 0$ so that $\ker(M) = \lambda c$, $\lambda \in \mathbb{R}$ and the solutions to the unconstrained problem can be written as
\begin{equation}
    f(u, v) = 0 \text{ if and only if } v = p^+u - b + \lambda c, \lambda \in \mathbb{R}.
\end{equation}

We now proceed to add constraints. The definitions of $q_y$ and $q_y^+$ imply that
\begin{align}
    u_k & \geq 0, \\
    u_k &\leq 1, \\
    1_K^\top u &= 1, \\
    v_k & \geq 0, \\
    v_k & \leq u_k, \\
    1_K^\top v & \leq 1.
\end{align}
Of these, $u_k \leq 1$, $u_k \geq 0$ and $1_K^\top v \leq 1$ are redundant since they are implied by the other constraints. From $1_K^\top u = 1$ we get together with $1_K^\top b = 0$ and $1_K^\top c = 1$ an equation for $\lambda$:
\begin{equation}
    1_K^\top v = p^+ + \lambda
\end{equation}
so that the set of solutions to $\deltasp = 0$ can be written as
\begin{align}\label{eq:b1}
    S &= \{ (u, v): f(u, v) = 0, 1_K^\top u = 1, 0 \leq v_k \leq u_k \} \\
    &= \{ (u, v): p^+u - v - b - (p^+ - 1_K^\top v) c = 0, 1_K^\top u = 1, 0 \leq v_k \leq u_k \} \\
    &= \{ (u, v): p^+(u-v) - (1-p^+)v + c 1_K^\top v = p^+ c + b, 1_K^\top u = 1, v_k \geq 0, u_k - v_k \geq 0 \}
\end{align}
which is our first result. We now proceed to apply Farkas' lemma to find conditions for the set of solutions to be non-empty.
From \eqref{eq:b1} we know that $(u, v) \in S$ if and only if
    \begin{align}\label{eq:c1}
        p^+(u-v) - (1-p^+)v + c 1_K^\top v &= p^+ c + b \\
        1_K^\top u &= 1 \\
        v &\geq 0, \\
        u-v &\geq 0.
    \end{align}
    For $i = 1, \ldots, 2K$ define vectors $s_i \in \mathbb{R}^{K+1}$ by
    \begin{align}
        s_i &\defeq -(1-p^+)e_i + c + e_{K+1}, \quad i = 1, \ldots, K \\
        s_{i+K} &\defeq p^+ e_i + e_{K+1}, \quad i = 1, \ldots, K
    \end{align}
    and let
    \begin{equation}
        d \defeq \begin{pmatrix}
            p^+ c + b \\
            1
        \end{pmatrix}.
    \end{equation}
    Then \eqref{eq:c1} is equivalent to
    \begin{align}\label{eq:c2}
        \sum_{i=1}^{2K} w_i s_i &= d, \\
        w_i & \geq 0
    \end{align}
    if we set $w_i = v_i, w_{i+K} = u_i - v_i, i = 1, \ldots, K$.
    By Farkas' lemma (see e.g. \cite{hiriart-urruty_fundamentals_2001}) the system \eqref{eq:c2} has a solution if and only if
    \begin{equation}
        S_1 \defeq \{ x \in \mathbb{R}^{K+1}: s_i^\top x \leq 0 \text{ for all $i$} \} \subseteq \{ x \in \mathbb{R}^{K+1}: d^\top x \leq 0 \} \defeqrev S_2.
    \end{equation}
    We have that $x \in S_1$ if and only if
    \begin{align}
      -(1-p^+) x_i + c_i + x_{K+1} &\leq 0 \label{eq:c:s1} \\
      p^+ x_i + x_{K+1} &\leq 0 \label{eq:c:s2}
    \end{align}
    and $x \in S_2$ if and only if
    \begin{equation}\label{eq:c:d}
        \sum_i (p^+ c_i + b_i) x_i + x_{K+1} \leq 0.
    \end{equation}
\end{proof}

\begin{proof}[Proof of Corollaries \ref{corollary:spsufficient} and \ref{corollary:spnecessary}]
    We keep the setting from the proof of Theorem~\ref{theorem:deltaspzero}. 
    In particular, if $c_i + \frac{b_i}{p^+} \geq 0$ for all $i$ then we have from \eqref{eq:c:s2} that for all $x \in S_1$
    \begin{equation}
        \sum_i (c_i + \frac{b_i}{p^+}) p^+ x_i + x_{K+1} 
        \leq \left(\sum_i (c_i + \frac{b_i}{p^+})\right) (-x_{K+1}) + x_{K+1} = 0
    \end{equation}
    since $\sum_i c_i = 1, \sum_i b_i = 0$, implying that $x \in S_2$.
    To solve for $c_i + \frac{b_i}{p^+} \geq 0$, write
    \begin{align}
        c_i + \frac{b_i}{p^+} 
        &= (1-p^+)\varphi_i^+ + p^+\varphi_i^- + (1-p^+)\frac{1-p_D}{p_D}(\varphi^+_i - \varphi^-_i) \\
        &= \varphi^-_i + \frac{1}{p_D}(1-p^+)(\varphi_i^+ - \varphi^-_i)
    \end{align}
    so that
    \begin{equation*}
        c_i + \frac{b_i}{p^+} \geq 0 \text{ if and only if } p_D \geq (1-p^+)\frac{\varphi^-_i - \varphi^+_i}{\varphi^-_i}
    \end{equation*}
    which gives our sufficient condition.
    On the other hand if $c_i + \frac{b_i}{p^+} < 0$ for all $i$ then
    \begin{equation}
        (c_i + \frac{b_i}{p^+})p^+ x_i > (c_i + \frac{b_i}{p^+})(-x_{K+1})
    \end{equation}
    so that
    \begin{equation}
        \sum_i (c_i + \frac{b_i}{p^+})p^+ x_i + x_{K+1} > 0
    \end{equation}
    and $S_1 \cap S_2 = \emptyset$, giving the necessary condition.
%    The corresponding conditions with upper indices $-$ instead of $+$ follow from the fact that the constraints in Theorem~\ref{theorem:deltaspzero} are invariant with respect to exchanging of the indices.
\end{proof}

\begin{proof}[Proof of Theorem~\ref{theorem:deltaeqoppzero}]
Recall the notation $q_y^- = q_y - q_y^+, p_y^- = p_y - p_y^+$.
Our goal is to characterize the values of $q_y, q_y^+$ that solve $\deltaeqopp(y) = 0$ under the constraints that $\sum_y q_y = 1, 0 \leq q_y^+ \leq q_y$. The last constraint can be replaced by $q_y = q_y^+ + q_y^-, q_y^+ \geq 0, q_y^- \geq 0$.
Let $w^\pm_y = \frac{q^\pm_y}{p^\pm_y}$. We know from \eqref{eq:deltaspplusminus} that
\begin{align}
    \deltaeqopp(y) &= \deltaeqodd(y, y) \\
    &= C_{y,y}^+ \left(1 - p_D \frac{q_y^+}{p_y^+}\right) - C_{y,y}^- \left(1 - p_D \frac{q_y^-}{p_y^-}\right) + p_D \left(\frac{q_y^+}{p_y^+} - \frac{q_y^-}{p_y^-}\right) \\
    &= C_{y,y}^+(1-p_D w_y^+) - C_{y,y}^-(1-p_D w_y^-) + p_D(w_y^+ - w_y^-).
\end{align}
Writing $w^{\pm}_y = p_D^{-1} - x^{\pm}_y$ and substituting, we have that
\begin{equation}
    \deltaeqopp(y) = 0 \text{ if and only if } (1-C_{y,y}^+)x_y^+ = (1-C_{y,y}^-)x_y^-.
\end{equation}
Denote $\lambda_y = (1-C_{y,y}^+)x_y^+ = (1-C_{y,y}^-)x_y^-$. From the relations $q_y = q_y^+ + q_y^- = p_y^+ w_y^+ + p_y^- w_y^-$ and $w^\pm_y = p_D^{-1} + \lambda_y / (1-C_{y,y}^\pm)$ we get, assuming $C_{y,y}^\pm \neq 1$,
\begin{equation}\label{eq:a1}
    q_y = \frac{p_y}{p_D} - \left(\frac{p_y^+}{1-C_{y,y}^+} + \frac{p_y^-}{1-C_{y,y}^-}\right)\lambda_y.
\end{equation}
Setting $\alpha_y = \frac{p_y^+}{1-C_{y,y}^+} + \frac{p_y^-}{1-C_{y,y}^-}$, the constraint that $\sum_y q_y = 1$ is equivalent to
\begin{equation}\label{eq:a2}
    \sum_y \alpha_y \lambda_y = \frac{1-p_D}{p_D}.
\end{equation}
At this point, we have satisfied $q_y = q_y^+ + q_y^-$ by definition in \eqref{eq:a1} and $\sum_y q_y = 1$ by imposing the constraint \eqref{eq:a2}. It remains to express $q^+ \geq 0, q^- \geq 0$ in terms of $\lambda$. Since 
\begin{equation}
    q_y^\pm = p_y^\pm w_y^\pm = p_y^{\pm}\left(\frac{1}{p_D} - \frac{\lambda_y}{1-C_{y,y}^+}\right)
\end{equation}
we finally get that the constraints $q_y^\pm \geq 0$ are equivalent to
\begin{align}
    \lambda_y & \leq \frac{1-C_i^+}{p_D} \\
    \lambda_y & \leq \frac{1-C_i^-}{p_D}
\end{align}
which completes the proof.
\end{proof}

\begin{proof}[Proof of Corollaries \ref{corollary:eqoppsufficient} and \ref{corollary:eqoppnecessary}]
From the constraint that $\sum_y \alpha_y \lambda_y = (1-p_D)/p_D$ and $\lambda_y \leq \frac{1-C_i^+}{p_D}$ we get that
\begin{equation}
    1-p_D = p_D \sum_y \alpha_y \lambda_y \leq p^+ + \sum_y p^- \frac{1-C_{y,y}^+}{1-C_{y,y}^-}
\end{equation}
and after rearranging, using that $(1-p^+) = p^- = \sum_y p_y^-$, we get the first necessary condition
\begin{equation}
    p_D \geq 1 - p^+ - \sum_y p_y^- \frac{1-C_{y,y}^+}{1-C_{y,y}^-}
    = \sum_y p_y^- (1 - \frac{1-C_{y,y}^+}{1-C_{y,y}^-}) \geq p^- \min_y \frac{C_{y,y}^- -C_{y,y}^+}{1-C_{y,y}^-}
\end{equation}
which is the result. We can repeat the same argument using $\lambda_y \leq \frac{1-C_i^-}{p_D}$ to get the second necessary condition.

On the other hand the same computation shows that if
\begin{align}
    p_D &\geq p^- \min_y \frac{C_{y,y}^- -C_{y,y}^+}{1-C_{y,y}^-} \text{ and } \\
    p_D &\geq p^+ \min_y \frac{C_{y,y}^+ -C_{y,y}^-}{1-C_{y,y}^+}
\end{align}
then
\begin{equation}
    \sum_y \alpha_y \lambda_y^{\textrm{max}} \geq \frac{1-p_D}{p_D}
\end{equation}
where
\begin{equation}
    \lambda_y^{\textrm{max}} = \max\left(\frac{1-C_i^+}{p_D}, \frac{1-C_i^-}{p_D}\right)
\end{equation}
so that by the intermediate value theorem there is some $\lambda^* \in \mathbb{R}^K$ with $\lambda^*_y \leq \lambda_y^{\textrm{max}}$ such that
\begin{equation}
    \sum_y \alpha_y \lambda^*_y \geq \frac{1-p_D}{p_D}    
\end{equation}
which is sufficient to satisfy all the constraints.
\end{proof}

\begin{proof}[Proof of Theorem~\ref{theorem:deltaeqoddzero}]
In order to have $\deltaeqodd(y, \yhat) = 0$ we need
\begin{align}
    (1-C_{y,\yhat}^+) w_y^+ &= (1-C_{y,\yhat}) w_y^- \text{ for all } y \neq \yhat \label{eq:d1} \\
    p_D (1-C_{y,y}^+) w_y^+ + C_{y,\yhat}^+ &= p_D (1-C_{y,y}^-) w_y^- + C_{y,y}^- \text{ for all } y = 0, \ldots, K. \label{eq:d2}
\end{align}
Equation \eqref{eq:d1} is only possible if the ratio $\frac{1-C_{y,\yhat}^+}{1-C_{y,\yhat}^-}, y \neq \yhat$ is independent of $\yhat$. Let us assume that this is the case and denote this ratio by $r_y$ so that
\begin{equation}
    r_y \defeq \frac{1-C_{y,\yhat}^+}{1-C_{y,\yhat}^-}, \quad \yhat \neq y
\end{equation}
and $w_y^- = r_y w_y^+$.
Then using that $\sum_{y'} C_{y,y'} = 1$ we can compute $r_y$ as follows:
\begin{align}
    1 - C_{y,y}^+ &= \sum_{\yhat \neq y} C_{y,\yhat}^+ = - \sum_{\yhat \neq y} (1-C_{y,\yhat}^+ - 1) \\
    &= -\sum_{\yhat\neq y} r_y (1-C_{y,\yhat}^-) + (K-1) \\
    &= -r_y (K-1) + r_y \sum_{\yhat\neq y} C_{y,\yhat}^- + (K-1) \\
    &= (1-r_y) (K-1) + r_y(1 - C_{y,y}^-).
\end{align}
By solving for $r_y$ we get that
\begin{equation*}
    r_y = \frac{(K-1) - (1-C_{y,y}^+)}{(K-1) - (1-C_{y,y}^-}).
\end{equation*}
Now we can substitute in \eqref{eq:d2} to compute $w_y^+$:
\begin{equation}
    p_D w_y^+ = \frac{C_{y,y}^- - C_{y,y}^+}{(1-C_{y,y}^+) - r_y (1-C_{y,y}^-)}.
\end{equation}
Finally from the requirement that $\sum_y q_y^+ + \sum_y q_y^- = 1$ we get
\begin{align}
    p_D 
    &= \sum_y p_D (p_y^+ w_y^+ + p_y^- w_y^-) \\
    &= \sum_y (p_y^+ + r p_y^-) p_D w_y^+ \\
    &= \sum_y \frac{(p_y^+ + r p_y^-)(C_{y,y}^- - C_{y,y}^+)}{(1-C_{y,y}^+) - r_y (1-C_{y,y}^-)}.
\end{align}
\end{proof}

% \begin{proof}[Proof of Theorem~\ref{theorem:simultaneous}]
% \TBD
% \end{proof}

\end{document}